\documentclass[journal]{article}

\usepackage{amsmath, amsfonts, amsthm}
\newtheorem{problem}{Problem}
\newtheorem{theorem}{Theorem}

\usepackage[top = 1in,left = 1in, right = 1in, bottom = 1in]{geometry}

\usepackage{algorithm, algorithmic}

\usepackage{graphicx}
\usepackage{subfigure}
\usepackage{tikz}

\begin{document}

\title{Smoothed Hierarchical Dirichlet Process: A Non-Parametric Approach to Constraint Measures}

\author{Cheng Luo,
        Yang Xiang
        and Richard Yi Da Xu}

\maketitle

\begin{abstract}
Time-varying mixture densities occur in many scenarios, for example, the distributions of keywords that appear in publications may evolve from year to year, video frame features associated with multiple targets may evolve in a sequence. Any models that realistically cater to this phenomenon must exhibit two important properties:  the underlying mixture densities must have an unknown number of mixtures; and there must be some ``smoothness'' constraints in place for the adjacent mixture densities. The traditional Hierarchical Dirichlet Process (HDP) may be suited to the first property, but certainly not the second. This is due to how each random measure in the lower hierarchies is sampled independent of each other and hence does not facilitate any temporal correlations. To overcome such shortcomings, we proposed a new Smoothed Hierarchical Dirichlet Process (sHDP). The key novelty of this model is that we place a temporal constraint amongst the nearby discrete measures $\{G_j\}$ in the form of symmetric Kullback-Leibler (KL) Divergence with a fixed bound $B$. Although the constraint we place only involves a single scalar value, it nonetheless allows for flexibility in the corresponding successive measures. Remarkably, it also led us to infer the model within the stick-breaking process where the traditional Beta distribution used in stick-breaking is now replaced by a new constraint calculated from $B$. We present the inference algorithm and elaborate on its solutions. Our experiment using NIPS keywords has shown the desirable effect of the model.

\end{abstract}

\section{Introduction}

The Hierarchical Dirichlet Process (HDP) is an extension to the traditional Dirichlet process (DP), which essentially is comprised of a set of measures responsible for generating observations in each of its respective groups. The model allows the lower hierarchies of different groups to share ``atoms" of their parents. They allow practitioners to apply the model to scenarios where independent mixture densities (of each group) share certain mixture components amongst its siblings which are inherited from their parents.

HDP is proposed by \cite{Cai2013} and further studied by \cite{Teh2008}\cite{Wang2011}\cite{Paisley2009}\cite{Hoffman2012} and many other literatures in various disciplines have used it as an infinite version of Latent Dirichlet Allocation (LDA) \cite{Blei2012}. In order to perform inference, \cite{Cai2013} proposed three Monte-Carlo Markov Chain (MCMC) methods, two of which adopted collapsed Gibbs with $G_0$ and $G_j$ integrated out while the third one instantiated $G_0$ and integrated $G_j$. The other works focus on the variational inference of HDP.

HDP has been successfully extended to model sequential data, in the form of so called HDP-HMM \cite{Cai2013}\cite{Wang2011}\cite{Hoffman2012} and infinite HMM \cite{Beal2002}.  In these models, it is assumed that there exists a set of ``time invariant measures''; A series of latent states are then drawn from these time invariant measures. The index for which a measure is to be used at time $t$, is determined by the state of the previous state, i.e., $p(z_t | z_{t-1}, {G_1, \dots G_k}) = G_{z_{t-1}}$.  In order to cater for a ``smooth'' transition of states, \cite{Fox2009} proposed the so-called `sticky' HDP-HMM method. This approach adds a weight for the self-transition bias and places a separate prior on this prior. While such methods can be simpler in terms of inference, and may make more sense in situations where the states distributions may repeat at some point in time, they  are not suitable for scenarios where the underlying distribution does not repeat in cycles, i.e., there is a need for a unique distribution to be assigned at each time $t$. One example of this is the distribution of topics for a scientific journal, it is highly unlikely that its topic distributions will ``come'' back at some years in the future. 

In order to generate a set of non-repeating and time-varying measures associated with observations at time $t$, \cite{Lin2010} proposed a Bayesian non-parametric approach in the form of Dependent Dirichlet processes. Based on an initial Dirichlet process, successive measures can be formed through three operations: superposition, subsampling and point transition. However, each of its operations can be very restrictive. For example, in sub-sampling, the ``next'' measure must have a fewer number of ``sticks'' than the first. The full flexibilities can only be achieved using combinations of all three operations, which makes the method very complex.
	
In this paper, we propose an alternative method to construct a time-varying dependent Dirichlet processes. Our model uses a very simple constraint in which the ``smoothness'' of the mixing measures in the second layer of HDP is achieved by simply placing a symmetric KL divergence constraint between them which is bounded by some fixed value $B$. The key motivation is achieved through the following observation: Using our KL constraint, we can achieve our intended outcome by substituting the Beta distribution with a truncated Beta while still using the stick-breaking paradigm. The new truncation is calculated from any $B$ valued placed, and we observed that there are only a finite number of possible solution spaces for the truncations. Subsequently, we developed a sampling method using a Gibbs Sampling framework, where one of the Gibbs steps is achieved using particle filtering. We named our method Smoothed HDP (sHDP). Since we only applied a simple scalar value for its constraint, we argue it is the non-parametric approach to place constraints amongst Dirichlet Processes. 

This rest of the paper is organised as follows: The second section describes the background of our model, and for completeness we include the related inference method. The third section is the elaborated description of our model. In section 4 we give the details of the inference method. In section 5, we test the model using synthetic data sets and the PAMI article key words. Section 6 provides a conclusion and some further discussions.

\section{Background}
	
To make this paper self-contained, we describe in this paper, some of the related models and their associated inference methods. We adopt the stick breaking paradigm for the representation of a Dirichlet process for both $G_0$ and $G_j$.
In terms of inference, given both $G_0$ and $G_j$, particle filtering is used instead of Gibbs sampling to infer $G_{j+1}$ since the conditional distribution, $\mathbb{P}(G_{j+1} | G_0, G_j )$ cannot be obtained analytically. We use the slice sampling proposed by \cite{Damien2001} for the sampling of the truncated Beta distribution.

	\subsection{The hierarchical Dirichlet process}
	
	Let $H$ be a diffusion measure (a measure having no atoms a.s.) \cite{Klenke2008} and $\gamma$ be a scalar, and a discrete measure $G_0$ is sampled from  $\mathrm{DP}(\gamma, H)$ \cite{Sethuraman1994}. Setting $G_0$ as the base measure and $\alpha>0$ be the concentration parameter, we sample $G_j \sim \mathrm{DP}(\alpha, G_0)$ for $j = 1,...,M$. Obviously, $G_j$ should be discrete and some of the activated atoms for different $G_j$ are shared. We the adopt stick-breaking paradigm for the representation of a Dirichlet process, and denote 
	\begin{equation}\label{eq: G_0}
		G_0 = \sum_{i =1}^\infty \pi_i \delta_{\phi_i},
	\end{equation}
	where $\delta$ is the Dirac function,
	$$
	\begin{aligned}
		\tilde{\pi}_i &\stackrel{i.i.d.}{\sim} \mathrm{Beta}(1, \gamma), \quad \pi_i = \tilde{\pi}_i\prod_{l=1}^{i-1}\tilde{\pi}_l, \\
		\phi_i &\stackrel{i.i.d.}{\sim} H,
	\end{aligned}
	$$
	and $\tilde{\pi}_i$ and $\phi_i$ are mutually independent. As proposed by \cite{Cai2013}, $G_j$ can be formulated as
	\begin{equation}\label{eq: G_j}
		G_j = \sum_{i = 1}^{\infty} \beta_i \delta_{\phi_i}, \quad \textrm{for } j = 1,...,M,
	\end{equation}
	where 
	\begin{align}\label{eq: sample beta origin}
		\tilde{\beta}_i \sim \mathrm{Beta}\left(\alpha \pi_i, \alpha\left( 1 - \sum_{l = 1}^{i} \pi_l \right) \right), \quad \beta_i = \tilde{\beta}_i \prod_{l=1}^{i-1}(1-\tilde{\beta}_l).
	\end{align}
	This formulation has an advantage in that the atoms in (\ref{eq: G_j}) are distinct. 
	
	\subsection{Posterior sampling of the GEM distribution}
	
The	Chinese restaurant process is famous in the posterior sampling of the Dirichlet process. However, since our model requires to sample $G_0$ and $G_j$ explicitly, we use the method proposed by \cite{Ishwaran2001}. Denoting by $m_i$ the number of observations equal to $i$, the posterior of $G_0$ is drawn as
	\begin{equation}\label{eq:sample G_0}
		\begin{aligned}
		\tilde{\pi}_i &\sim \mathrm{Beta}\left(1 + m_i, \gamma + \sum_{l = i+1}^{\infty} m_l\right), \\
		\pi_i &= \tilde{\pi}_i \prod_{l = 1}^{i-1} (1 - \tilde{\pi}_l).
		\end{aligned}
	\end{equation}

	\subsection{Slice sampling of a truncated Beta distribution}
	
	For the sampling of the truncated Beta distribution, we use slice sampling proposed by \cite{Damien2001}. Let $0 \leq a < b \leq 1$ and $\alpha, \beta > 0$ be scalars and the density of $X$ is 
	$$
	f(x|a,b,\alpha,\beta) = \frac{\Gamma(\alpha+\beta)}{\Gamma(\alpha)\Gamma(\beta)}x^{\alpha-1}(1-x)^{\beta-1}\delta(a<x<b),
	$$
    where $\delta(\cdot)$ is the Dirac function.To sample from this truncated Beta distribution, we add auxiliary variable $U$ and construct a joint density of $X, U$ as
	$$
	p(x,u) \propto x^{\alpha - 1}\delta(0 < u < (1 - x)^{\beta - 1}, a < x < b).
	$$
	The method of drawing of $x$ and $u$ is: 
	\begin{enumerate}
		\item Sampling $u$ conditional on $x$ 
		$$
		u \sim \mathrm{Uniform}(0, (1 - x)^{\beta - 1}).
		$$
		\item Sampling $x$ conditional on $u$ using inverse transform
		\begin{enumerate}
			\item if $\beta > 1$
			$$
			p(x) \propto x^{\alpha - 1}\delta(a < x < \min\{1 - u^{1/(\beta-1)}, b\}),
			$$
			\item if $\beta < 1$
			$$
			p(x) \propto x^{\alpha - 1}\delta(\max\{1 - u^{1/(\beta - 1)},a\} < x < b).
			$$
		\end{enumerate}		
	\end{enumerate}
	
	\subsection{Particle filtering}
	
	Particle filtering is used to sample the posterior distribution of sequential random variables. The advantage of particle filtering is that sampling is possible when the posterior density cannot be stated analytically. Let $\{z_t:t = 1,...,T\}$ be latent random variables and $\{x_t:t=1,...,T\}$ be observed random variables. The joint distribution of them is
	$$
	p(x_{1:t}, z_{1:t}) \propto p(z_1)\prod_{t=1}^{T-1} p(z_{t+1}|z_t) \prod_{t=1}^T p(x_t|z_t).
	$$ 
	Let $N$ be a positive integer large enough, and the posterior sampling method is:
	\begin{enumerate}
	    \item For $i = 1,...,N$, propose $z^{(i)}_1$ from prior $p(z_1)$, and compute weights $w^{(i)} = p(x_1|z_1)$.
	    \item Resample $\boldsymbol{z}_1$ with weights $\boldsymbol{w}$, where $\boldsymbol{z}_1$ denotes the vector of $(z^{(i)}_1)_{i=1,...,N}$ and $\boldsymbol{w}$ denotes $(w^{(i)})_{i=1,...,N}$.
	    \item For $t=2,...,T$,
	    \begin{enumerate}
	    	\item propose $z^{(i)}_t$ from $p(z_{t+1}|z_t)$, and compute weights $w^{(i)} = p(x_t|z_t)$.
	    	\item resample $\boldsymbol{z}_t$ with weights $\boldsymbol{w}$.
	    \end{enumerate}
	\end{enumerate}
	The empirical posterior distribution of $(z_t)_{t=1,...,T}$ is 
	\begin{equation}\label{eq:particle posterior}
		p(z_{1:T}) = \frac{1}{N}\sum_{i=1}^{N} \delta_{z^{(i)}_{1:T}}.
	\end{equation}
	For a fully description of particle filtering please refer to \cite{Doucet2011}.

\section{The model of the smoothed HDP}
	
	In this section, we will give a full description of our model. In the traditional HDP, the discrete measures in the lower level, i.e., $G_j$, are independent given the concentration parameter  $\alpha$ and the base measure $G_0$. Separately from this independent assumption,  we force the successive mixing measures $G_{j}$ and $G_{j+1}$ to be alike in terms of their symmetric KL divergence. Since both $G_j$ and $G_{j+1}$ are discrete measures with infinite many atoms, the computation of symmetric KL divergence between them is intractable. As a substitution, we propose an aggregated form of symmetric KL divergence that is computable. It can be proved that our aggregated symmetric KL divergence has the same expectation as the original form although they are not equal all the time.

	\subsection{The aggregated symmetric KL divergence}
	
	Let $\gamma,\alpha > 0$ be scalars and $H$ be a diffusion measure, we sample a discrete base measure $G_0$ distributed as  $\mathrm{DP}(\gamma, H)$. Then the mixing measures ($G_j$) in the lower level are sampled from 
	$$
	G_j \sim \mathrm{DP}(\alpha, G_0), \quad \textrm{for }j = 1,...,M,
	$$ 
	with a constraint that the symmetric KL divergence $\mathrm{KL}(G_{j+1}||G_j) < B$ for some fixed positive scalar $B$. Suppose $G_0$ and $G_j$ have the following expression
	$$
	\begin{aligned}
	G_0 =\sum_{i=1}^{\infty} \pi_i \delta_{\phi_i}, \quad 
	G_j = \sum_{i=1}^{\infty} \beta_{j,i} \delta_{\phi_i}.
	\end{aligned}
	$$
	The symmetric KL divergence between $G_j$ and $G_{j+1}$ is defined to be
	\begin{small}
	\begin{equation}\label{eq: symm KL}
		\mathrm{KL}(G_{j}||G_{j+1}) = \sum_{i=1}^{\infty} \beta_{j,i} \log \frac{\beta_{j,i}}{\beta_{j+1,i}} + \sum_{i=1}^{\infty} \beta_{j+1,i} \log \frac{\beta_{j+1,i}}{\beta_{j,i}}.
		\end{equation}
	\end{small}
	
	Our problem is summarised as 
    \begin{problem}\label{problem : 1}
    	Given $G_0$ and $G_j$, 
    	$$
    	\begin{aligned}
    	&\textrm{sample } & &G_{j+1} \sim \mathrm{DP}(\alpha, G_0),\\
    	&\textrm{subject to } & & \mathrm{KL}(G_{j}||G_{j+1}) \leq B,
    	\end{aligned}
    	$$
    	where $B > 0$ is a scalar.
    \end{problem}
    
    However, since (\ref{eq: symm KL}) is a sum of infinite terms, the direct computation of (\ref{eq: symm KL}) is intractable. Alternatively, for every atom $l$, we define an aggregated symmetric KL divergence as
    \begin{equation}\label{eq: aggregate kl}
    \begin{aligned}
    &~~~~\mathrm{aggKL}(l;G_j||G_{j+1})\\
     &= \left(\sum_{i=1}^{l-1} \beta_{j,i}\right) \log \frac{\sum_{i=1}^{l-1}\beta_{j,i}}{\sum_{i=1}^{l-1}\beta_{j+1,i}}  \\
     &~~+\beta_{j,l}\log \frac{\beta_{j,l}}{\beta_{j+1,l}} + \left(\sum_{i=l+1}^{\infty} \beta_{j,i}\right) \log \frac{\sum_{i=l+1}^{\infty}\beta_{j,i}}{\sum_{i=l+1}^{\infty}\beta_{j+1,i}} \\
    & ~~+ \left(\sum_{i=1}^{l-1} \beta_{j+1,i}\right) \log \frac{\sum_{i=1}^{l-1}\beta_{j+1,i}}{\sum_{i=1}^{l-1}\beta_{j,i}} + \beta_{j+1,l}\log \frac{\beta_{j+1,l}}{\beta_{j,l}}  \\
    & ~~+\left(\sum_{i=l+1}^{\infty} \beta_{j+1,i}\right) \log \frac{\sum_{i=l+1}^{\infty}\beta_{j+1,i}}{\sum_{i=l+1}^{\infty}\beta_{j,i}}.
    \end{aligned}	
    \end{equation}    
    Hence Problem \ref{problem : 1} changes to 
    \begin{problem}\label{problem : 2}
    	Given $G_0$ and $G_j$, for $l = 1,2,3,...$
    	\begin{equation} \label{codition 2}
   	    \begin{aligned}
    	&\textrm{sample } & & G_{j+1} \sim \mathrm{DP}(\alpha, G_0), \\
    	&\textrm{subject to} & & \mathrm{aggKL}(l;G_{j}||G_{j+1}) \leq B^*,
    	\end{aligned}
    	\end{equation}

    	where $B^* > 0$ is a scalar.
    \end{problem}
	
	The simplest way to set the value of $B^*$ is letting $B^* = B$. But this does not make sense unless $\mathrm{KL}(G_j||G_{j+1}) = \mathrm{aggKL}(G_j||G_{j+1})$. However, direct algebraic calculation shows that the these two symmetric KL divergences are not agree at all times. Fortunately, the next theorem shows that the two terms,  $\mathrm{KL}(G_j||G_{j+1})$ and $\mathrm{aggKL}(G_j||G_{j+1})$,   take the same value on average. 
	
	\begin{theorem}\label{thm: 1}
		The expectation of $\mathrm{KL}(G_j||G_{j+1})$ and $\mathrm{aggKL}(l;G_{j}||G_{j+1})$ agrees with respect to $\boldsymbol{\beta}_j, \boldsymbol{\beta}_{j+1}$, where $\boldsymbol{\beta}_j = (\beta_{j,i})_{i=1,2,...}$ and $\boldsymbol{\beta}_{j+1} = (\beta_{j+1,i})_{i=1,2,...}$.
	\end{theorem}
	\begin{proof}
	Subtract (\ref{eq: aggregate kl}) from (\ref{eq: symm KL})  gives
	\begin{small}
    \begin{align}
    	&~~~~\mathrm{KL}(G_j||G_{j+1}) - \mathrm{aggKL}(l;G_j||G_{j+1}) \label{eq: kl minus proof} \\
    	&= \sum_{i=1}^{l-1} \left( \beta_{j,i}\log \frac{\beta_{j,i}}{\sum_{i=1}^{l-1}\beta_{j,i}} -\beta_{j+1,i}\log \frac{\beta_{j+1,i}}{\sum_{i=1}^{l-1}\beta_{j+1,i}}\right) \nonumber  \\
    	&~~~+ \sum_{i=l+1}^{\infty} \left( \beta_{j,i}\log \frac{\beta_{j,i}}{\sum_{i=l+1}^{\infty}\beta_{j,i}} -\beta_{j+1,i}\log \frac{\beta_{j+1,i}}{\sum_{i=l+1}^{\infty}\beta_{j+1,i}}\right). \nonumber
    	\end{align}
	\end{small}
	Since $\beta_{j,i}$ and $\beta_{j+1,i}$ has the same distribution for all $i = 1,2,3,...$, (\ref{eq: kl minus proof}) must has expectation $0$ provides $\log \sum_{i=l+1}^{\infty}\beta_{j,i}$ and $\sum_{i=l+1}^{\infty} \beta_{j+1,i}$ are finite. But this is obvious since $\sum_{i=l+1}^{\infty}\beta_{j+1,i} = 1 - \sum_{i=1}^{l}\beta_{j,i}$ and $\sum_{i=l+1}^{\infty}\beta_{j+1,i} = 1 - \sum_{i=1}^{l}\beta_{j+1,i}$.
	\end{proof}
	In light of Theorem \ref{thm: 1}, we are safe to set $B = B^*$. The details of  sampling  $G_{j+1}$ is given in the next subsection.
    
	\subsection{How to sample $G_{j+1}$ conditional on $G_j$}
	
	Suppose $G_0$, $G_j$ and $\beta_{j+1,1:l-1}$ is known, and we want to sample $\beta_{j+1,l}$.  Observe (\ref{eq: aggregate kl}) and with a few algebraic calculation, the aggregated symmetric KL divergence constraint can be restated as 
	\begingroup
	\footnotesize
	\begin{equation}\label{eq: inequality}
	\begin{aligned}
	&- \left( 1 - \sum_{i=1}^{l}\beta_{j,i} \right) \log \left( 1 - \sum_{i=1}^{l-1} \beta_{j+1,i} - \beta_{j+1,l} \right)  \\
	& + \left( 1 - \sum_{i=1}^{l-1}  \beta_{j+1,l} -\beta_{j+1,i}\right) \log \frac{1 - \sum_{i=1}^{l-1}\beta_{j+1,i} - \beta_{j+1,l}}{\sum_{i=l+1}^{\infty}\beta_{j,i}}\\
	&\beta_{j+1,l}\log \frac{\beta_{j+1,l}}{\beta_{j,l}} -\beta_{j,l} \log \beta_{j+1,l} \leq B - C,
	\end{aligned}
	\end{equation} 
	\endgroup
	where $C$ is  a constant with respect to $\beta_{j+1,l}$, which is
	$$
	\begin{aligned}
	C &= \left(\sum_{i=1}^{l-1} \beta_{j,i}\right) \log \frac{\sum_{i=1}^{l-1}\beta_{j,i}}{\sum_{i=1}^{l-1}\beta_{j+1,i}} + \beta_{j,l} \log \beta_{j,l} \\
	&+ \left(\sum_{i=1}^{l-1} \beta_{j+1,i}\right) \log \frac{\sum_{i=1}^{l-1}\beta_{j+1,i}}{\sum_{i=1}^{l-1}\beta_{j,i}}.
	\end{aligned}
	$$
	Note the function
	\begingroup
	\small
	$$
	\begin{aligned}
	f(x)  &=x\log \frac{x}{\beta_{j,l}} -\beta_{j,l}\log x \\
	& -  \left( 1 - \sum_{i=1}^{l}\beta_{j,i} \right) \log \left( 1 - \sum_{i=1}^{l-1} \beta_{j+1,i} - x \right)  \\
	& + \left( 1 - \sum_{i=1}^{l-1}  \beta_{j+1,l} -x\right) \log \frac{1 - \sum_{i=1}^{l-1}\beta_{j+1,i} - x}{\sum_{i=l+1}^{\infty}\beta_{j,i}}
	\end{aligned}
	$$
	\endgroup
	is convex in the interval $\left(0, 1 - \sum_{i=1}^{l-1}\beta_{j+1,i}\right)$, the equality
	\begin{equation}\label{eq: equality}
		f(x) = B - C
	\end{equation}
	has at most two roots. Showing in Figure \ref{fig: solutions}, inequality (\ref{eq: inequality}) has solution of form
	\begin{itemize}
		\item $\beta_{j+1,l} \in (r_1, r_2)$ if (\ref{eq: equality}) has two solutions $r_1$ and $r_2$,
		\item $\beta_{j+1,l} \in (0,r_1)$ or $(r_1,1)$ if (\ref{eq: equality}) has only one solution $r_1$.
		\item If (\ref{eq: equality}) has no roots, either the entire interval $\left(0, 1 - \sum_{i=1}^{l-1}\beta_{j+1,i}\right)$ is the solution or $B$ is too tight that there is no discrete measure satisfies condition (\ref{codition 2}).
	\end{itemize}
	
	 According to the relationship of $\beta_{j+1,l}$ and $\tilde{\beta}_{j+1,l}$ shown in (\ref{eq: G_j}), the truncating boundary for $\tilde{\beta}_{j+1,l}$ are $r_1 / \prod_{i=1}^{l-1}(1 - \tilde{\beta}_{j+1,i})$ and  $r_2 / \prod_{i=1}^{l-1}(1 - \tilde{\beta}_{j+1,i})$, where $r_1$ is the lower bound and $r_2$ is the upper bound and they can be $0$ or $1$. Combine this constraint and the sampling method of $\beta_{j}$ in (\ref{eq: G_j})(\ref{eq: sample beta origin}) gives the solution of Problem \ref{problem : 2} in Algorithm \ref{algo: 1}.
	
	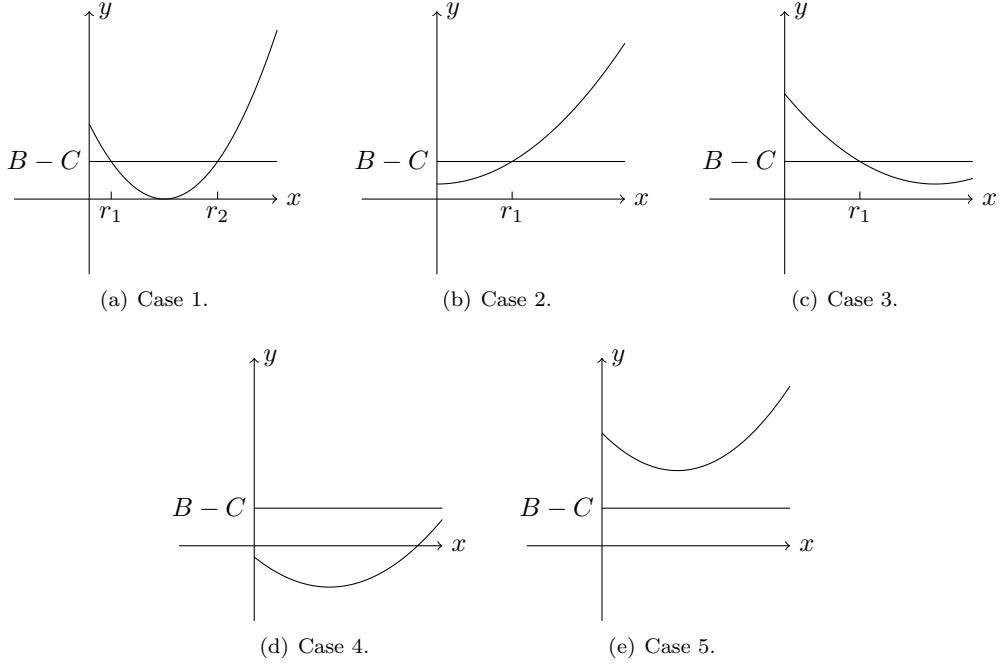
\begin{figure*}
		\centering
		\begin{subfigure}[Case 1.]
		{	
		
		\begin{tikzpicture}
			\draw[->] (-1,0) -- (2.5,0) node [right] {$x$};
			\draw[->] (0,-1) -- (0,2.5) node [right] {$y$};
			\draw[scale=1,domain=0:2.5,smooth,variable=\x,black] plot ({\x},{(\x-1) * (\x - 1)});
			\draw (0,0.5) -- (2.5,0.5);
			\node at (1.7071,-0.2) {$r_2$}; 
			\node at (0.2929,-0.2) {$r_1$};
			\draw (1.7071,0) -- (1.7071,0.1);
			\draw (0.2929,0) -- (0.2929,0.1);
			\node at (-0.6, 0.5) {$B-C$};
			\end{tikzpicture}
		}
		\end{subfigure}
		\begin{subfigure}[Case 2. ]
			{
				\begin{tikzpicture}
				\draw[->] (-1,0) -- (2.5,0) node [right] {$x$};
				\draw[->] (0,-1) -- (0,2.5) node [right] {$y$};
				\draw[scale=1,domain=0:2.5,smooth,variable=\x,black] plot ({\x},{(\x) * (0.3 *\x) + 0.2});
				\node at (-0.6, 0.5) {$B-C$};
				\draw (0,0.5) -- (2.5,0.5);
				\draw (1,0) -- (1,0.1);
				\node at (1,-0.2) {$r_1$};
				\end{tikzpicture}
			}
		\end{subfigure}
		\begin{subfigure}[Case 3. ]
			{
				\begin{tikzpicture}
				\draw[->] (-1,0) -- (2.5,0) node [right] {$x$};
				\draw[->] (0,-1) -- (0,2.5) node [right] {$y$};
				\draw[scale=1,domain=0:2.5,smooth,variable=\x,black] plot ({\x},{0.3 * (\x-2) * (\x - 2) + 0.2});
				\node at (-0.6, 0.5) {$B-C$};
				\draw (0,0.5) -- (2.5,0.5);
				\draw (1,0) -- (1,0.1);
				\node at (1,-0.2) {$r_1$};
				\end{tikzpicture}
			}
		\end{subfigure}
		\begin{subfigure}[Case 4. ]
			{
				\begin{tikzpicture}
				\draw[->] (-1,0) -- (2.5,0) node [right] {$x$};
				\draw[->] (0,-1) -- (0,2.5) node [right] {$y$};
				\draw[scale=1,domain=0:2.5,smooth,variable=\x,black] plot ({\x},{0.4*(\x-1) * (\x - 1)-0.55});
				\node at (-0.6, 0.5) {$B-C$};
				\draw (0,0.5) -- (2.5,0.5);
				\end{tikzpicture}
			}
		\end{subfigure}
		\begin{subfigure}[Case 5. ]
			{
				\begin{tikzpicture}
				\draw[->] (-1,0) -- (2.5,0) node [right] {$x$};
				\draw[->] (0,-1) -- (0,2.5) node [right] {$y$};
				\draw[scale=1,domain=0:2.5,smooth,variable=\x,black] plot ({\x},{0.5 * (\x-1) * (\x - 1) + 1});
				\node at (-0.6, 0.5) {$B-C$};
				\draw (0,0.5) -- (2.5,0.5);
				\end{tikzpicture}
			}
		\end{subfigure}
	\caption{The cases of solution (\ref{eq: inequality}). The feasible interval for case (a)-(e) is $(r_1,r_2)$, $(0,r_1)$, $(r_1,1)$, $(0,1)$, $\emptyset$ respectively.} \label{fig: solutions}
	\end{figure*}
	
	\begin{algorithm}
		\caption{Sampling $G_{j+1}$ with aggregated symmetric KL divergence constraint.} \label{algo: 1}
		\begin{algorithmic}[1]
			\REQUIRE Base measure $G_0$, predecessor $G_j$, concentration parameter $\alpha$ and the symmetric KL divergence bound $B$.
			\ENSURE A discrete measure $G_{j+1}$ satisfying condition (\ref{codition 2}).
			
			\FOR {$l = 1,2,3,...$}
				\STATE Derive the roots of (\ref{eq: equality}). 
				\STATE Compute truncating boundary for $\tilde{\beta}_{j+1,l}$.
		     	\STATE Sample $\tilde{\beta}_{j+1,l}$ from a truncated Beta
		     	 distribution.
			\ENDFOR
			\FOR {l = 1,2,3,...}
			    \STATE Compute $\beta_{j+1,l}$.
			\ENDFOR
		\end{algorithmic}
	\end{algorithm}

	\section{Inference}
	
	Note in Algorithm \ref{algo: 1} we sample $G_{j+1}$ directly without the explicit  likelihood, hence we cannot compute the posterior of $G_{j+1}$ analytically. Consequently,  we have to use particle filtering to sample the posterior distribution of $G_{j}$ for $j = 1,...,M$.

	\subsection{Initialization}
	
	According to the definition of Dirichlet process \cite{Ferguson1973}, given concentration parameter $\alpha$ and base measure $G_0$, the expectation of  distributions sampled from $\mathrm{DP}(\alpha, G_0)$ is $G_0$. Let $d_{j,h}$ be an observation in phase $j$, then the distribution of $d_{j,h}$  is modeled as
	\begin{align}
		d_{j,h} \sim G_j, \quad G_j \sim \mathrm{DP}(\alpha, G_0).
	\end{align} 
	Followed by the deduction above and integrate out $G_j$, the distribution of $d_{j,h}$ is
	\begin{align}
		d_{j,h} \sim \int G_j(d_{j,h}) dG_{j} = G_0.
	\end{align}
	Hence the posterior distribution of $G_0$ can be sampled with help of (\ref{eq:sample G_0}). Suppose the prior of $G_0$ is a Dirichlet process $\mathrm{DP}(\gamma, H)$ and $G_0$ has the form of (\ref{eq: G_0}). Let  $c_i$ be a set of the observations taking on cluster $i$ and $m_i = |c_i|$, then the posterior of $\tilde{\pi}_i$ is
	\begin{equation*}
		p(\tilde{\pi}_i|c_{i},c_{i+1},..., \gamma) \sim \mathrm{Beta}\left(1 + m_i,\gamma + \sum_{l=i+1}^\infty m_l\right),
	\end{equation*} 
	and the posterior of $\pi_i$ is 
	\begin{equation*}
		\pi_i = \tilde{\pi}_i \prod_{l=1}^{i-1}(1 - \tilde{\pi}_l).
	\end{equation*}
	The posterior of $\phi_i$ is 
	\begin{equation*}
		p(\phi|c_i, d_{j,\cdot}) \propto H(\phi_i)\prod_{h \in c_i}F(d_{j,h}|\phi_i),
	\end{equation*}
	where $d_{j,\cdot}$ denote all the observations in phase $j$ and $F(\cdot)$ is the likelihood function.
	
	\subsection{Update of $G_j$}
	
	Given $G_0$, $G_j$ and $\alpha$, we use particle filtering to sample the posterior of $G_j$. With the help of Algorithm \ref{algo: 1}, we give a proposal $G_{j+1}$. The weight of this proposal is a product of probabilities
	\begin{equation} \label{eq: wights}
		w = \prod_{i} \beta_i^{m_{j,i}},
	\end{equation} 
	where $m_{j,i}$ is the number of observations in phase $j$ that take cluster $i$. Using particle filtering, our sampling algorithm of $G_{1:M}$ is given in Algorithm \ref{algorithm 2}.
	
	\begin{algorithm}
		\caption{Particle filtering for $G_{1:M}$} \label{algorithm 2}
		\begin{algorithmic}[1]
			\REQUIRE The base measure $G_0$, concentration parameter $\alpha$, symmetric KL divergence bound $B$, observations $d$ and a large integer $N$.
			\ENSURE Discrete measures $G_{1:M}$ satisfying condition (\ref{codition 2}) for each pair of successive $G_j$ and $G_{j+1}$.
			
			\STATE For $n = 1,...,N$, sample $G_1$ by (\ref{eq: G_j})(\ref{eq: sample beta origin}).
			\FOR {$j = 2,...,M$}
				\STATE For $n = 1,...,N$, sample $G_j^{(n)}$ with Algorithm \ref{algo: 1}.
				\STATE Compute importance weights $w_j^{(n)}$ by (\ref{eq: wights}).
				\STATE Normalize importance wights.
				\STATE Resample $G_j^{(n)}$ with importance weights.
			\ENDFOR
		\end{algorithmic}
	\end{algorithm}
	
	\subsection{Update of $z$}
	
	The latent variable $z_{j,h}$ is used to indicate which cluster the observation $d_{j,h}$ is from. By the definition of $G_j$, the probability of $z_{j,h} = i$ is $\beta_{j,i}$. Given $F(\cdot)$ and $\phi_i$, the likelihood is $F(d_{j,h}|\phi_i)$. Combining these together gives the posterior probability of $z_{j,h}$, which is proportional to
	\begin{equation*}
		p(z_{j,h} = i| \phi_i, \beta_{j,i}) \propto \beta_{j,i} F(d_{j,i}|\phi_i),
	\end{equation*}
	where $\beta_{j,i}$ is the weight of the atom $\phi_i$  in $G_j$.
	
	\subsection{Update of the hyper-parameters $\gamma, \alpha$}
	
	For hyper-parameters $\gamma$ and $\alpha$, we simply put Gamma distribution as  priors on them, namely, $\mathrm{Gamma}(a_\gamma, b_\gamma)$ and  $\mathrm{Gamma}(a_\alpha, b_\alpha)$. The posterior of $\gamma$ and $\alpha$ are
	\begin{small}
	\begin{align*}
			p(\gamma|a_\gamma, b_\gamma, k, m) &\propto \gamma^{a_\gamma - 2} e^{-b_\gamma \gamma}(\gamma + m)\int_0^1 u^\gamma (1 - u)^{m-1} du,\\
			p(\alpha|a_\alpha, b_\alpha, k, m) &\propto \alpha^{a_\alpha - 2} e^{-b_\alpha \alpha}(\alpha + m)\int_0^1 u^\alpha (1 - u)^{m-1} du.
		\end{align*}
	\end{small}

	respectively, where $m$ is the number of observations and $k$ is the number of distinct atoms. By adding auxiliary variable $u$, it renders a mixture of Gammas for $\gamma$ and $\alpha$.
	For details of the sampling method please refer to \cite{West1992}.
	
\section{Experiments}
	
	We validate our smoothed HDP model using both synthetic and real dataset. MATLAB code of our experiments are available in the following website: https://github.com/llcc402/MATLAB-codes.
		
	We assume the number of activated atoms is at most $K = 100$, hence the definition of the symmetric KL divergence (\ref{eq: symm KL}) becomes:
	\begin{small}
	\begin{equation} \label{eq: kl practical}
			\mathrm{KL}(G_{j}||G_{j+1}) = \sum_{i=1}^{K} \beta_{j,i} \log \frac{\beta_{j,i}}{\beta_{j+1,i}} + \sum_{i=1}^{K} \beta_{j+1,i} \log \frac{\beta_{j+1,i}}{\beta_{j,i}}.
		\end{equation}
	\end{small}

	In case some weights of $G_j$ become close to zero, causing the KL to become arbitrarily large, we set a minimum lower bound for every component weight to be $\epsilon = 1\mathrm{e}-5$.
	
	\subsection{Simulations on synthetic data set}
	
	Since the computation of the symmetric KL divergence only requires the knowledge of weights of the discrete measures, we discard the positions and sample $\mathrm{GEM}(\gamma)$ instead. In our experiments, we set $\gamma = 5$. Conditioning on $G_0$, we sample $G_1 \sim \mathrm{DP}(\alpha, G_0)$, we then sample $G'_2 \sim \mathbb{P}(G_2 | \alpha, G_0, G_1)$ using our sHDP. We compare our model with HDP in order to show its smoothness effect. In HDP, while having $G_0$ and $G_1$ sampled in the same way as sHDP, $G_2$ is instead to be generated independently, i.e., $G_2 \sim \mathrm{DP}(\alpha, G_0)$. In all the experiments we set $\alpha = 1$.
	
	In sHDP, we set the KL bound $B = 3$, and compare the symmetric KL divergence between $G_1, G_2$ and $G_1, G'_2$. We repeat this procedure for $1000$ times and compare the symmetric KL divergence by (\ref{eq: kl practical}), and the result is shown in Figure \ref{fig: shdp v.s. hdp}. From this figure, we can see that the symmetric KL divergence of successive measures are much smaller than the traditional model. The mean of the values of sHDP is about $3$, just as Theorem \ref{thm: 1} illustrated, while the mean of traditional HDP is about $10$.
	\begin{figure*}
		\centering
		\begin{subfigure}[A comparison of the symmetric KL divergence for sHDP and HDP. Left is the smoothed sampling result and right is the traditional sampling result.] 
		{ \label{fig: shdp v.s. hdp}
			\includegraphics[height = 120pt, width = 240pt]{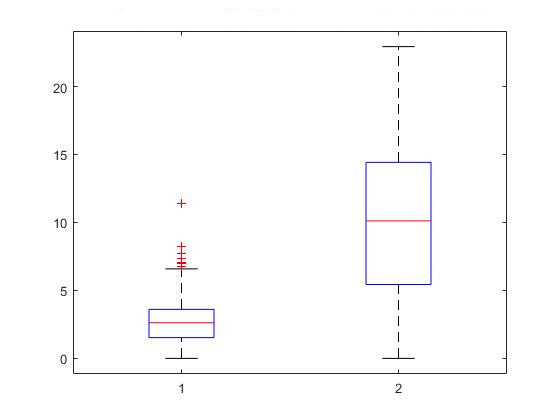}
		}
		\end{subfigure}
		\begin{subfigure}[The KL divergences sHDP when the bound $B$ changes. From left to right the bound $B$ is set from $1$ to $10$.]
		{\label{fig: change bound}
			\includegraphics[height = 120pt, width = 240pt]{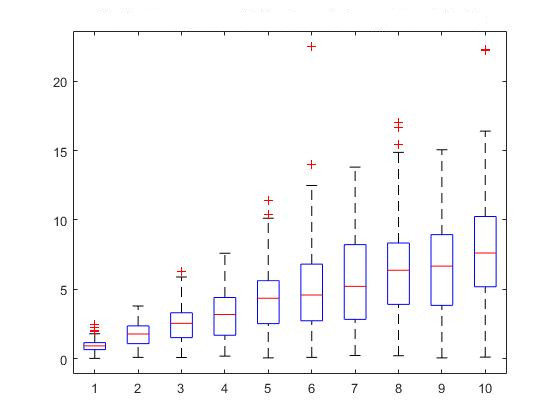}
		}
		\end{subfigure}
		\caption{Comparing the symmetric KL divergences. }
	\end{figure*}
	
	In the second experiment, we show how the mean of KL varies with the bound $B$ we choose. By setting $B$ from $1$ to $10$, we sample $G_1, G'_2$ $100$ times and compute the symmetric KL divergence between them. This is shown in Figure \ref{fig: change bound}. We show that the average symmetric KL divergence of the two distributions have increased when the bound value increases, which leads to an expected outcome.
	
	The last simulation is about using the time series data with the following setting: We first set $\gamma = 5$, $\alpha = 1$ and $B = 1$, then we generate $G_1,...,G_{20}$ using sHDP. After that, we add Gamma noise to the distributions $G_1,...,G_{20}$.  The noise is simulated as follows. First generate $K \times 20$ independent random variables distributed as $\mathrm{Gamma}(0.03, 1)$, then add them to $G_1,...,G_{20}$. Lastly, we normalize the distributions. With these noisy distributions, we generate $50$ positive integers from each of them representing the component indices which are seen as the observations of the model in order to infer the posterior of $G_{1:20}$. We sample the posterior from both HDP and sHDP. The symmetric KL divergence of successive distributions is shown in Figure \ref{fig: successive kl}. It can be seen that the symmetric KL values of sHDP are below $1$ while that of HDP are much larger, from about $3$ to $7$. We also compared the distance between the theoretical value of $G_{1:20}$ and the posterior of both models in which the symmetric KL divergence is used as the measure of distance. It can be seen that in the first $10$ time intervals, the distances are similar, however, in the last $10$ sHDP results much smaller KL distance than that of the HDP.
	
	\begin{figure*}
		\begin{subfigure}[The symmetric KL divergence for successive distributions. The blue solid line is the values of the KL divergence for sHDP, and the red dotted line is the values of the KL divergence for HDP.]
		{\label{fig: successive kl}
			\includegraphics[height = 120pt, width = 250pt]{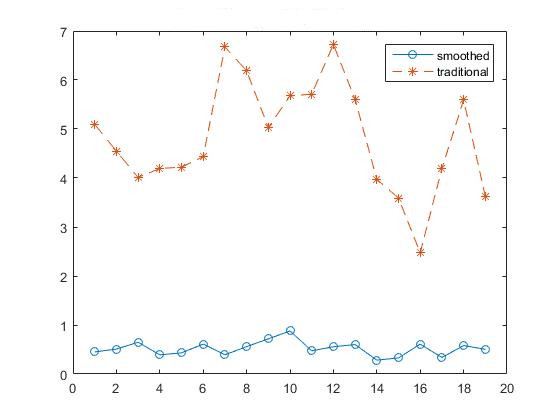}
		}
		\end{subfigure}
		\begin{subfigure}[The distance between theoretic value $G_{1:20}$ with the posterior. The blue solid line is the distance between theoretical $G_{1:20}$ and the posterior of sHDP, and the red dotted line is the distance between theoretical $G_{1:20}$ and the posterior of  HDP.]
		{\label{fig: kl_theoretic_post}
			\includegraphics[height = 120pt, width = 250pt]{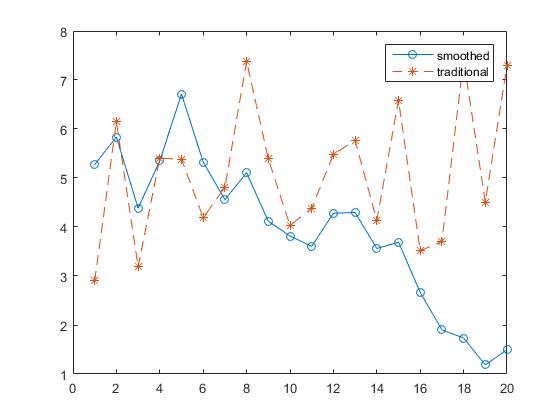}
		}
		\end{subfigure}
		\caption{The comparison of the successive KL divergences and the accuracy for the phased data set.}
	\end{figure*}
	
	\subsection{Applications on real data set}
	
	 We parsed the web page of PAMI and collected keywords from approximately $4000$ papers published from 1990 to 2015. The keywords from one paper is seen as a document and the years are considered to be the phases. Similarly to \cite{Lin2010}, we transform each of the documents into a $12$ dimensional vector using the method proposed by \cite{Jordan2004}. We compute the similarity between two documents, which is the number of shared words divided by the total number of words in these two documents. Then we derive $12$ eigen vectors for the normalized graph Laplacian \cite{Planck2006} and the $12$ column matrix of the eigen vectors are seen as the data set. For pre-processing, we change the standard deviation of the columns to be $1$.
	 
	 We assume the likelihood of the data is normally distributed with fixed variance and performed a hierarchical clustering of the data set. The clusters are considered to be the topics of the corpus. In the experiment, we set $\gamma = \alpha = 5$, and the KL bound $B = 3$, and we sample the posterior for $500$ iterations. The mixing measures $G_{1:26}$ represents the weights of the keyword clusters for each year. 
	 
	 When we compare the outcomes from both sHDP and HDP, we found that the clusters for each year are similar for both models, but the value of successive symmetric KL divergence can be quite different. The mean of the successive symmetric KL of sHDP is mostly blow $3$ while the mean of HDP is around $5$. The boxplot of the symmetric KL divergence is shown in Figure \ref{fig: pami}. Moreover, we track the components weights for the $8$ most significant clusters of $G_0$ through $G_1$ to $G_{26}$ in figure \ref{fig: fluc}. In this figure, we can see that when the weights of $G_0$ are large (the first $5$ sub-figures in Figure \ref{fig: fluc}), both HDP and sHDP show similar degree of fluctuation in its KL divergence. However, for less significant weight components of $G_0$, (the last few $3$ figures), the smoothing effect of sHDP can be more obviously shown. This illustrates the fact that sHDP can suppress smaller components over times, and hence  help achieve smoothness. 
	    
	 \begin{figure}
	 	\centering
	 	\includegraphics[height = 120pt, width = 260pt]{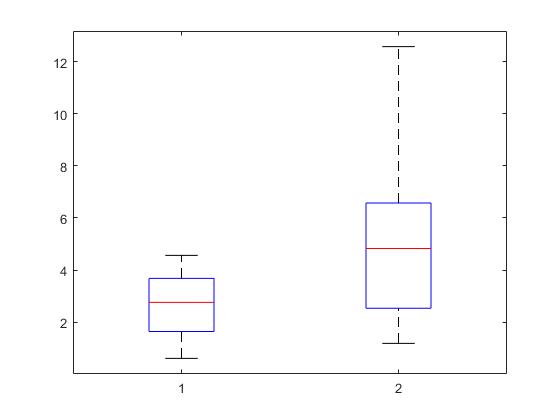}
	 	\caption{The symmetric KL of successive measures for sHDP and HDP. Left is sHDP and right is HDP.}\label{fig: pami}
	 \end{figure}
	 
	 \begin{figure*}
	 \centering
	 \begin{subfigure}
	 {
	 \includegraphics[height = 90pt, width = 130pt]{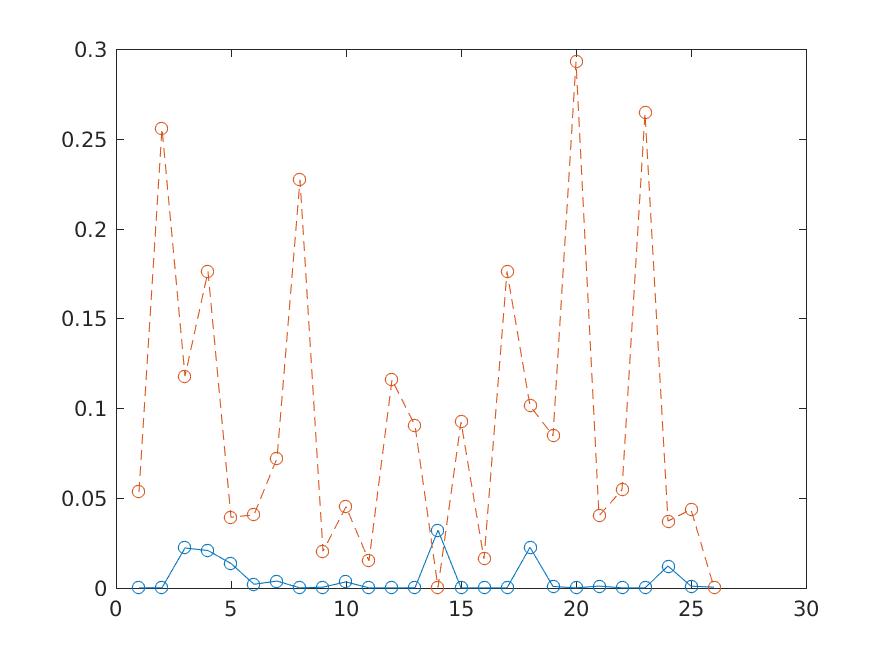}
	 }
	 \end{subfigure}
	 \begin{subfigure}
	 {
	 \includegraphics[height = 90pt, width = 130pt]{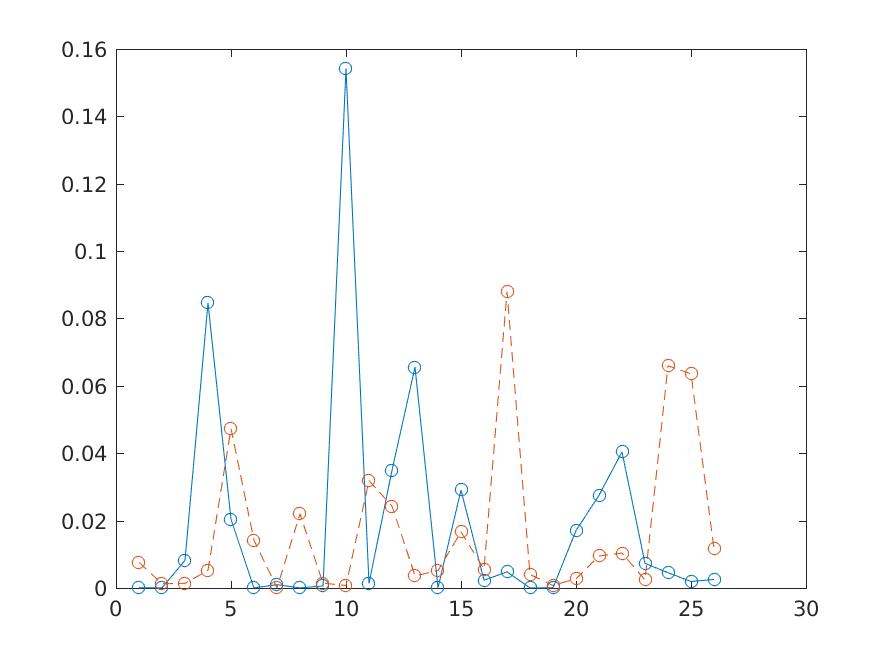}
	 }
	\end{subfigure}
	 \begin{subfigure}
	 {
	 \includegraphics[height = 90pt, width = 130pt]{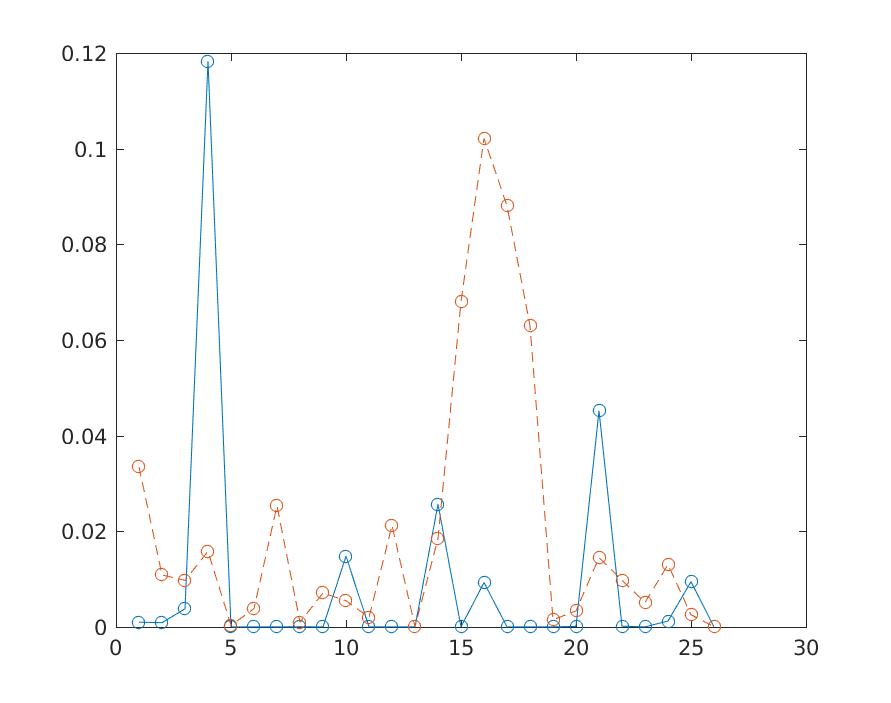}
	 }
	 \end{subfigure}
	 \begin{subfigure}
	 {
	 \includegraphics[height = 90pt, width = 130pt]{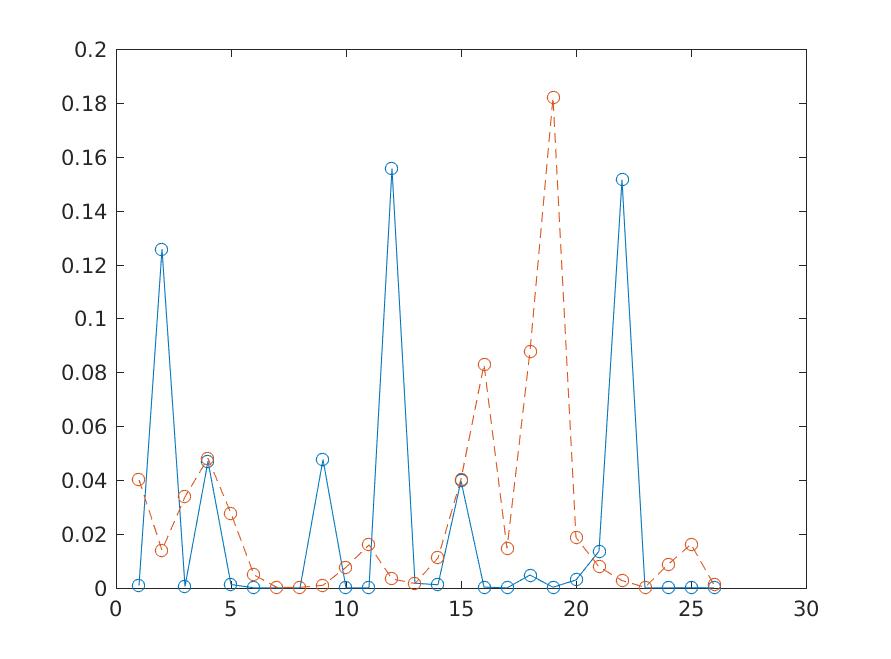}
	  }
	  \end{subfigure}
	 \begin{subfigure}
	 {
	 \includegraphics[height = 90pt, width = 130pt]{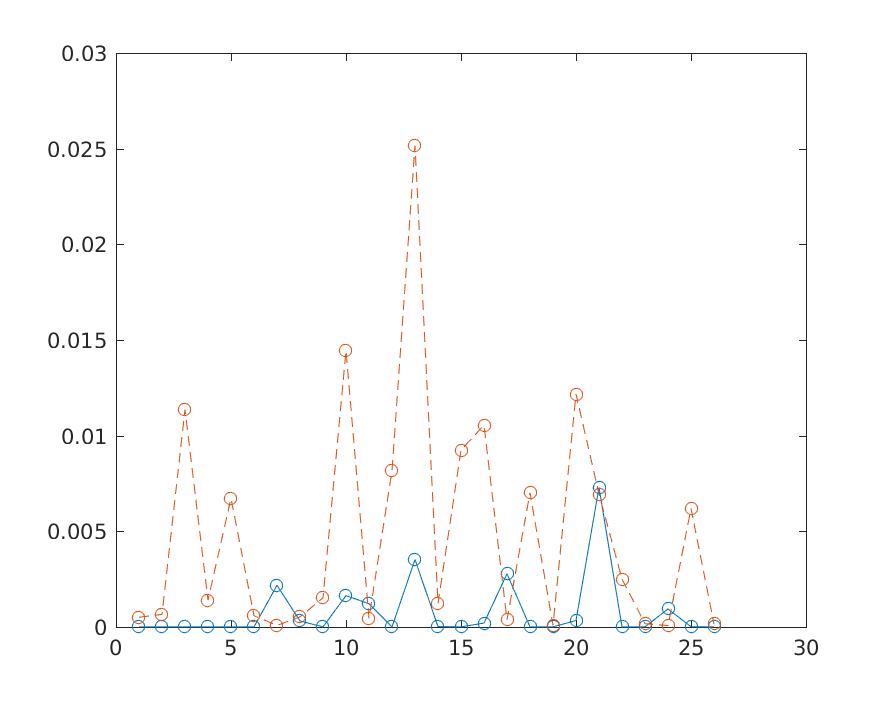}
	 }
	  \end{subfigure}
	  \begin{subfigure}
	 {
	 \includegraphics[height = 90pt, width = 130pt]{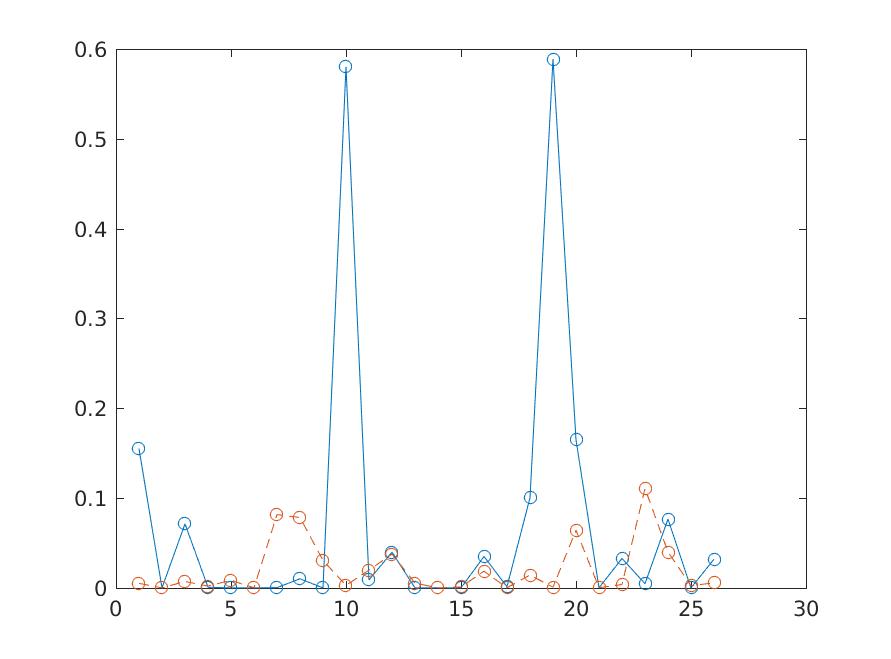}
	 }
	 \end{subfigure}
	 \begin{subfigure}
	  {
	 \includegraphics[height = 90pt, width = 130pt]{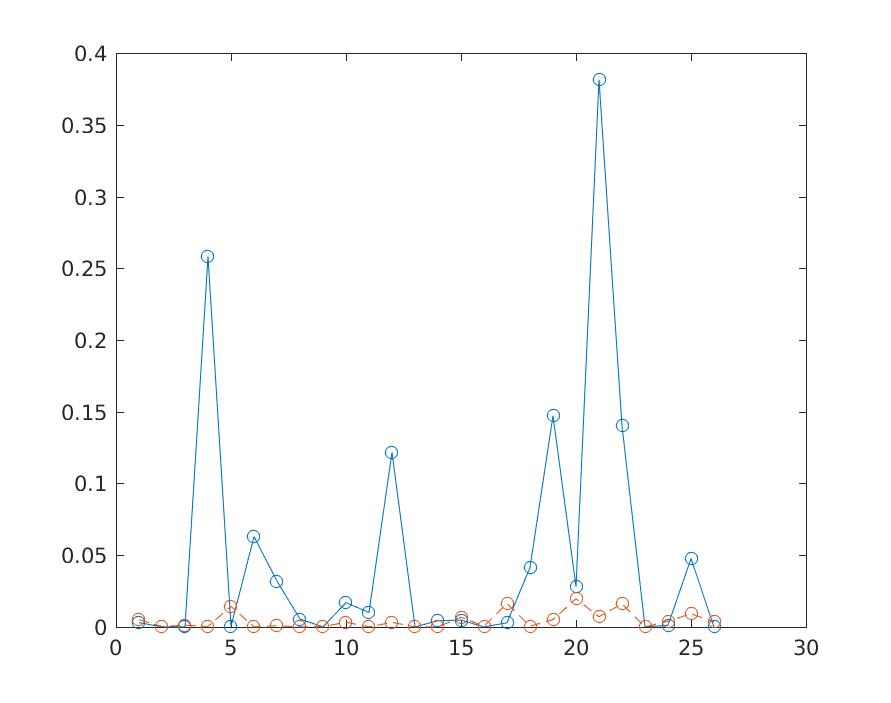}
	 }
	 \end{subfigure}
	 \begin{subfigure}
	 {
	 \includegraphics[height = 90pt, width = 130pt]{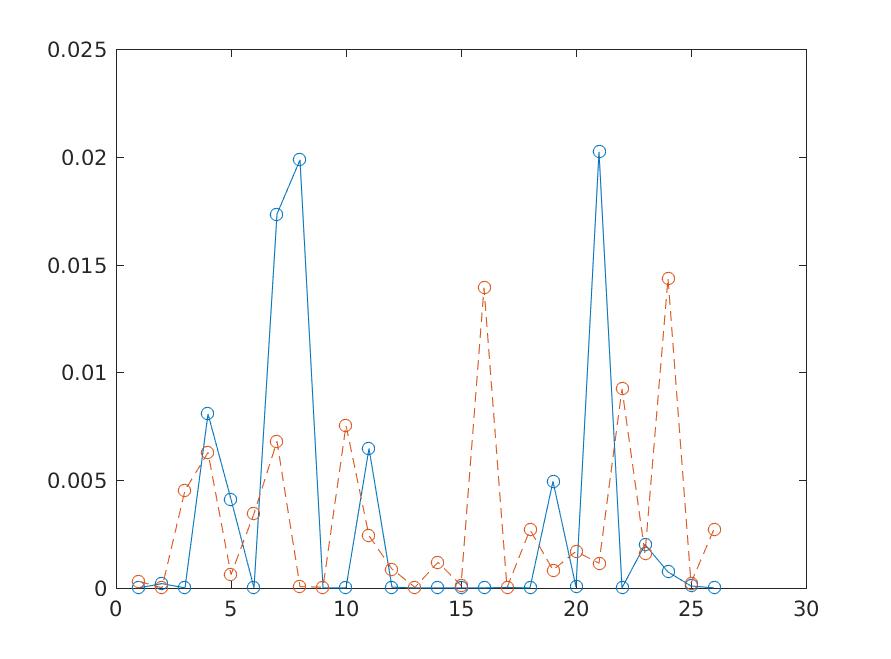}
	 }
	 \end{subfigure}
	 \caption{The fluctuation of the weights in $G_2$. From upper left to lower right are corresponding to the clusters with weights in $G_0$ in the descending order. The red dotted lines represent the weights of $G_2$ using sHDP and the blue solid represent the weiths of $G_2$ using HDP.}\label{fig: fluc}
	 \end{figure*}
	 
	\section{Conclusion and discussion}
	
We proposed a smoothed Hierarchical Dirichlet Process (sHDP), in which we add a smoothness constraint to the lower layer of a Dirichlet Process (DP), in order to cater to many settings where temporal and/or spatial smoothness between the underlying measures is required. We used symmetric KL divergence for this constraint. Although it is conceptually sound, the traditional symmetric KL in its original form can not be practically applied since the symmetric KL divergence between two countable discrete measures is an infinite sum of real numbers. For this reason, in our paper, we proposed to use aggregated symmetric KL divergence as an alternative measure. We proved that this substitution has the same expectation of the original symmetric KL divergence, and hence can be used appropriately as an alternative bound.

We showed that given a measure at time $t-1$, i.e., $G_j$, a base measure $G_0$ and the concentration parameter $\alpha$, the sampling of the weight of $G_{j+1}$, i.e.,  $\beta_{j+1,i}$ can be derived by using a truncated Beta distribution; We show that the truncation boundaries can be solved as roots of a convex function. This has made the solution space easily obtainable.

We show empirically the effect of the bound over the distribution $G_{j+1}$ in figure \ref{fig: change bound}: we vary the bound value from $1$ to $10$ for which we show that the average symmetric KL divergence of the sampled successive distributions have also increased accordingly, which is in line with our expectation. The application of our model to the real dataset shows the smoothness effect of our model, which can be illustrated in both figure \ref{fig: pami} and \ref{fig: fluc}. The box plot in \ref{fig: pami} show that KL divergence between successive measures, $G_j$ and $G_{j+1}$ of our model is much less than that of the HDP. 

We also have analysed the results further to gain some insights into the component-wise smoothness: We choose the eight most significant components of the base measure $G_0$, and then  tracked how its weight changes from $G_1$ to $G_{26}$ (as we have 26 years of data). The results are shown in figure \ref{fig: fluc}. It is interesting to note that we have a gradual increase in smoothness when its corresponding base measure component is decreased in  value.

We also like to point out that our proposed method actually violates de Finetti's theorem about the exchangeable random variables \cite{Ardus1985}. According to de Finetti's theorem, the prior exists only if the random variables are infinitely exchangeable. However, our assumption on the dependence of successive distributions actually made the prior/base measure $G_0$ non-existent.

In the current form, the non traceability of the conditional $\mathbb{P}(G_{j+1}| G_j, G_0 )$ require us to use particle filter as its inference method. Although it suffices in the settings of our testing dataset which is comprised of only 26 time intervals, it nonetheless can be a computation bottleneck when the application requires us to have much more granular time intervals. Therefore, in our future work, we will experiment with stochastic inference methods for time sequences, similar to that of \cite{Matthew2014}.

	\bibliographystyle{plain}
	
	\bibliography{library}

\end{document}